\documentclass[11pt]{article}
\usepackage[utf8]{inputenc}
\usepackage{amsgen,amsmath,amstext,amsbsy,amsthm,amsopn,amssymb,bbm,listings,float}
\usepackage{fullpage}

\usepackage{graphicx}
\usepackage{tcolorbox}
\usepackage[font={footnotesize}]{caption}
\usepackage[english]{babel}
\usepackage{amsthm}

\usepackage{tikz}

\begingroup
\makeatletter
   \@for\theoremstyle:=definition,remark,plain\do{%
     \expandafter\g@addto@macro\csname th@\theoremstyle\endcsname{%
        \addtolength\thm@preskip\parskip
     }%
   }
\endgroup

\usepackage[numbers,sort]{natbib}
\usepackage{algorithm}% http://ctan.org/pkg/algorithm
\usepackage{algpseudocode}
\newtheorem{theorem}{Theorem}[section]
\newtheorem{corollary}{Corollary}[theorem]
\newtheorem{lemma}[theorem]{Lemma}

\newtheorem{definition}{Definition}[section]

\newtheorem{claim}[theorem]{Claim}

\usepackage{hyperref}
\usepackage{tikz}
\hypersetup{
    colorlinks=true,
    linkcolor=blue,
    filecolor=magenta,
    urlcolor=cyan,
    citecolor = blue
}

\newcommand{\eps}{\varepsilon}
\newcommand{\A}{\mathcal{A}}
\newcommand{\M}{\mathcal{M}}
\newcommand{\R}{\mathcal{R}}
\newcommand{\T}{\mathcal{T}}
\newcommand{\X}{\mathcal{X}}
\newcommand{\Y}{\mathcal{Y}}
\newcommand{\Z}{\mathcal{Z}}
\renewcommand{\P}[2]{\mathbb{P}_{#1}\left[ #2 \right]}
\newcommand{\E}[2]{\mathbb{E}_{#1}\left[ #2 \right]}

\newcommand{\Ber}[1]{\ensuremath{\mathsf{Ber}}\left( #1 \right)}

\newcommand{\trar}{\pi}

\def \epsilon {\varepsilon} % make epsilon consistent
\newcommand{\poly}[1]{\text{poly}\left(#1\right)}
\newcommand{\hl}{\mathcal{HL}}
\newcommand{\cc}{\mathrm{CC}}
\newcommand{\sac}{\mathrm{SC}}
\newcommand{\fiub}{\textsc{HLSolver}}
\newcommand{\siub}{\textsc{PCSolver}}

\newcommand{\pcp}{\mathcal{PC}}
\newcommand{\non}{\text{N}}
\newcommand{\seq}{\text{S}}
\newcommand{\full}{\text{F}}
\newcommand{\bsc}{\text{BSC}}
\newcommand{\sgn}{\text{sgn}}
\newcommand{\overbar}[1]{\mkern 1.5mu\overline{\mkern-1.5mu#1\mkern-1.5mu}\mkern 1.5mu}

\newif\ifcomment
\commenttrue

\ifcomment
\newcommand{\mj}[1]{\textcolor{magenta}{[MJ: #1]}}
\newcommand{\jm}[1]{\textcolor{cyan}{[JM: #1]}}
\newcommand{\ar}[1]{\textcolor{violet}{[AR: #1]}}
\else
\newcommand{\mj}[1]{}
\newcommand{\jm}[1]{}
\newcommand{\ar}[1]{}
\fi

\title{Exponential Separations in Local Differential Privacy}
\author{Matthew Joseph \thanks{University of Pennsylvania, Computer and Information Science. \href{mailto:majos@cis.upenn.edu}{\texttt{majos@cis.upenn.edu}}} \and Jieming Mao \thanks{Google Research New York. This work done at the Warren Center, University of Pennsylvania. \href{mailto:maojm@google.com}{\texttt{maojm@google.com}}} \and Aaron Roth \thanks{University of Pennsylvania, Computer and Information Science. \href{mailto:aaroth@cis.upenn.edu}{\texttt{aaroth@cis.upenn.edu}}}}

\begin{document}

\maketitle
\begin{abstract}
	We prove a general connection between the \emph{communication} complexity of two-player games and the \emph{sample} complexity of their multi-player locally private analogues. We use this connection to prove sample complexity lower bounds for locally differentially private protocols as straightforward corollaries of results from communication complexity. In particular, we 1) use a communication lower bound for the hidden layers problem to prove an exponential sample complexity separation between sequentially and fully interactive locally private protocols, and 2) use a communication lower bound for the pointer chasing problem to prove an exponential sample complexity separation between $k$-round and $(k+1)$-round sequentially interactive locally private protocols, for every $k$.
\end{abstract}
\thispagestyle{empty} \setcounter{page}{0}
\clearpage

\section{Introduction}
\label{sec:intro}
A \emph{differentially private}~\cite{DMNS06} algorithm must induce similar output distributions given similar input databases. Intuitively, this similarity creates enough uncertainty to hide the presence or absence --- and hence ensure the privacy --- of any one individual in the dataset. Differential privacy's mathematical and practical appeal has driven the creation of a large body of work (see~\citet{DR14} and ~\citet{V17} for surveys) and increasing industrial and governmental adoption~\cite{A17, EPK14, BEMMR+17, DKY17, A16}. Here, we focus on the strictly stronger notion of \emph{local} differential privacy \cite{DMNS06, BNO08, KLNRS11}. In local differential privacy, the rows of the size-$n$ database are distributed among $n$ individual users, and the private computation consists of a public interaction between the users that must reveal little about any single datum.

Local differential privacy has been studied in three increasingly general models. In \emph{noninteractive} protocols, all users publish their messages simultaneously\footnote{In this paper, we will still allow some coordination in the form of shared randomness, which only makes our lower bounds stronger.}. In \emph{sequentially interactive}~\cite{DJW13} protocols, users publish exactly one message each in sequence, and messages may depend on previously-published messages. The number of rounds in such protocols is necessarily upper bounded by the number of users, $n$, but can be fewer. \emph{Fully interactive} protocols have no such restrictions. Each user may publish arbitrarily many messages with arbitrary dependencies on other previous messages, and there is no upper bound on the number of rounds. In this paper, we study the relative power of these models.

A natural way to approach this question is through the lens of sample complexity: fixing one statistical estimation problem, how many users $n$ (with data drawn i.i.d. from a fixed but unknown distribution) are necessary for a protocol to solve the problem in each model? It is known that there are problems that require exponentially larger sample complexity in noninteractive protocols than sequentially interactive protocols \cite{KLNRS11}. Before our work, however, only polynomial separations were known between sequentially and fully interactive protocols \cite{JMNR19}, and it was conjectured that the power of these two models was indeed polynomially related.

\subsection{Our Contributions}
First, we prove a general equivalence between the \emph{communication complexity} of any two-player problem and the \emph{sample complexity} of a sequentially interactive locally private multi-player analogue of that problem. Informally, we show that the ``noise'' that must be added to ensure local differential privacy in any sequentially interactive protocol makes it possible to convert it into a two-party protocol over a noisy channel, and vice-versa. In combination with past work relating communication complexity over noisy and noiseless channels~\cite{S96, BR14, BM15}, this conversion extends to noiseless channels as well.

\begin{theorem}[Informal version of Theorem~\ref{thm:eq}]
	Let $P_2$ be a two-player communication problem, let $P_m$ be its multi-player analogue, let $\cc_{\gamma}(P_2)$ be the communication complexity of solving $P_2$ with error $\gamma$, and let $\sac_{\gamma}^{\eps,S}(P_m)$ be the sample complexity of solving $P_m$ with error $\gamma$ using a  sequentially interactive $\eps$-locally private protocol. Then for any $\eps = O(1)$ and $0 < \gamma, \eta = \Omega(1)$ such that $\gamma + \eta < 1$, $\sac_\gamma^{\eps,S}(P_m) = \Theta\left(\tfrac{1}{\eps^2} \cdot \cc_{\gamma + \eta}(P_2)\right)$.
\end{theorem}

This connection enables us to translate existing communication lower bounds into sample complexity lower bounds for locally private protocols. For example, we show an exponential separation in sample complexity between sequentially and fully interactive locally private protocols for the \emph{hidden layers} problem.

\begin{theorem}[Informal version of Corollary~\ref{cor:lb} and Theorem~\ref{thm:fi_ub}]
	Let $k$ be a natural number parameterizing the hidden layers problem and $\eps = O(1)$. Any sequentially interactive $\eps$-locally private protocol solving the hidden layers problem with constant probability requires $\Omega\left(\tfrac{2^k}{\eps^2}\right)$ samples. However, there exists a fully interactive $\eps$-locally private protocol that solves the hidden layers problem with constant success probability and $O\left(\tfrac{k}{\eps^2}\right)$ samples.
\end{theorem}

Theorem \ref{thm:eq} provides a tight equivalence between the sample complexity of locally private sequentially interactive protocols, and communication complexity --- but the reduction does not preserve the number of rounds of interaction needed. We prove a different variant of Theorem \ref{thm:eq} that preserves round complexity at the expense of a looser dependence on $\eps$. Using this connection and round-specific communication lower bounds for the \emph{pointer chasing} problem, we show that for every $k$, there is an exponential separation between $(k-1)$- and $k$- round sequentially interactive protocols.

\begin{theorem}[Informal version of Corollary~\ref{cor:pc_ni_lb} and Theorem~\ref{thm:si_ub}]
	Let $k$ and $\ell$ be natural numbers parameterizing the pointer chasing problem. Any $(k-1)$-round sequentially interactive $\eps$-locally private protocol solving the pointer chasing problem with constant probability requires $\Omega\left(\tfrac{\ell}{e^\eps k^2}\right)$ samples. However, there exists a $k$-round sequentially interactive $\eps$-locally private protocol that solves pointer chasing with constant success probability and $\tilde O\left(\tfrac{k\log(\ell)}{\eps^2}\right)$ samples, where $\tilde O(\cdot)$ hides logarithmic factors.
\end{theorem}

The special case of $k=2$ provides another exponential separation between noninteractive and sequentially interactive protocols (such a separation was already known, for a different problem using a custom analysis of \citet{KLNRS11}). For $k>2$, the result is new and demonstrates that there is an infinite hierarchy of sequentially interactive protocols, characterized by their number of rounds of interactivity, and that there is an exponential separation between each level of the hierarchy.

Finally, for simplicity we state all of our results for pure $\eps$-local privacy. However, for reasonable values of $\delta$ (roughly $\delta = o\left(\tfrac{\eps}{n\log(n)}\right)$) they easily extend to $(\eps,\delta)$-local privacy using the approximate-to-pure transformation described by Bun et al.~\cite{BNS18} and Cheu et al.~\cite{CSUZZ18}.

\subsection{Related Work}
\label{subsec:rel}
\citet{KLNRS11} first separated noninteractive from sequentially interactive locally private protocols. They exhibited a learning problem, masked parity, for which a sequentially interactive locally private protocol requires only $O(\poly{d})$ samples while any noninteractive locally private protocol requires $2^{\Omega(d)}$ samples. Their result relies on an (up to polynomial terms) equivalence between the number of local randomizer calls made by a locally private protocol and the query complexity of a statistical query (SQ) algorithm. For sequentially interactive protocols, the number of local randomizer calls is equal to the sample complexity. However, this is no longer true for fully interactive protocols, and so their  equivalence between the sample complexity of locally private protocols and the query complexity of SQ algorithms fails to extend beyond sequential interactivity.~\citet{DF18} also leveraged this connection to SQ learning to prove similar separations for a different, larger class of problems.

Our results do not rely on this connection to the SQ model. In fact, our exponential separation between the sequential and fully interactive models implies that the connection fails to extend beyond sequential interactivity: there are problems that locally private protocols can solve with polynomial sample complexity, but SQ protocols require exponential query complexity. This re-focuses attention on an important problem that remains open: what can be learned privately in the local model of computation, if full interactivity is allowed?

~\citet{DJW13} first distinguished sequential and full interactivity and developed several lower bound techniques that applied to noninteractive and sequentially interactive protocols.~\citet{DR19} extended some of these techniques to full interactivity, thus establishing the first lower bounds for fully interactive locally private protocols that do not also extend to the centralized model of privacy. However, their results do not separate sequentially interactive protocols from fully interactive protocols.

~\citet{JMNR19} characterized the relationship between the sample complexity of fully interactive and sequentially interactive locally private protocols in terms of a parameter called ``compositionality'', providing the first separations between full and sequential interactivity. While this characterization is tight in terms of the compositionality parameter, their lower bound viewed as a function of $n$ only shows a sample complexity gap on the order of $\Omega(\sqrt{n})$ between the two models. We prove a separation that is exponential in $n$.

Several other works have also examined interactivity in local differential privacy.~\citet{STU17} showed that convex optimization protocols relying on neighborhood-based oracles (and, in particular, locally private protocols relying on neighborhood-based oracles) require interaction. However, their results are qualitatively different, as they also hold absent any privacy restriction.~\citet{ACFT19} showed polynomial separations between noninteractive private- and public-randomness protocols for identity testing (informally, in their formulation public randomness is necessary to coordinate the randomizers employed by different users). We note that our framework assumes public randomness for all noninteractive protocols; this only strengthens our separation.

Finally, our work is not the first to connect differential privacy and two-player communication games.~\citet{MMPRTV10} introduced and studied two-party differential privacy. They showed, roughly, that a two-party differentially private protocol is equivalent to a protocol with low round and communication complexity. This equivalence relies on work relating the compressibility of a protocol to its information cost. In contrast, we work in the more commonly studied $n$-party model of local differential privacy, and our reductions rely on different techniques.

\subsection{Organization}
Preliminaries from communication complexity and differential privacy appear in Section~\ref{sec:prelims}. We prove the equivalence between two-party communication complexity and multi-party locally private sample complexity in Section~\ref{sec:eq_red} and use this equivalence to separate sequential and full interactivity in Section~\ref{sec:si_fi}. A looser round complexity-preserving version of this connection appears in Section~\ref{sec:round_specific}, and we apply this second connection to show a round-specific separation in Section~\ref{sec:round_sep}. 

\section{Preliminaries}
\label{sec:prelims}
We start by recalling preliminaries from two-party communication complexity before moving on to differential privacy and (our notion of) multi-party communication complexity.

\subsection{Two-party Setting}
\label{subsec:two_party}
Many of our results feature the \emph{two-party communication model}.

\begin{definition}
	In the \emph{two-party communication model}, one player Alice receives input $x \in \X$, and the other player Bob receives input $y \in \Y$. Alice and Bob want to jointly compute some output $z \in \Z$ such that $(x,y,z)$ satisfies some relation $\R \subset \X \times \Y \times \Z$.
\end{definition}

To compute $z$, Alice and Bob coordinate their actions using a \emph{protocol}.

\begin{definition}
	Given a two-party communication model, a \emph{protocol} $\A$ specifies a binary output function that each player should apply to their data at each time step, as a function of the time step, the \emph{transcript} of previously released values and any shared randomness.
\end{definition}

Note that, by our inclusion of shared randomness, all protocols we consider are ``public-coin''. Two salient protocol characteristics are the number of bits that must be exchanged and the likelihood of a ``good'' outcome.

\begin{definition}
	The \emph{communication complexity} of a protocol $\A$, denoted by $\cc(\A)$, is the maximum number of bits exchanged over all inputs $(x,y)$ and all random coins. If on all inputs $x$ and $y$ $\P{\A}{(x,y,z) \not \in \mathcal{R}} \leq \gamma$, we say $\A$ \emph{computes $\R$ with error at most $\gamma$}. The \emph{randomized communication complexity with error $\gamma$} of a relation $\R$ is then defined as $$\cc_{\gamma} (\R) = \min_{\A:\A \text{ computes } \R \text{ with error at most } \gamma} \cc(\A).$$
\end{definition}

\subsection{Differential Privacy}
\label{subsec:dp}
As originally formulated by~\citet{DMNS06}, we say an algorithm $\M$ satisfies differential privacy if its output distribution is Lipschitz continuous with respect to a neighboring relationship on inputs:

\begin{definition}
	Given data universe $\X$, two datasets $S, S' \in \X^n$ are \emph{neighbors} if they differ in at most one coordinate. Given $\eps, \delta \geq 0$, a randomized algorithm $\M \colon \X^n \to \mathcal{O}$ is $(\eps,\delta)$-differentially private if for every pair of neighboring databases $S$ and $S'$ and every event $\Omega \subset \mathcal{O}$ $$\P{\M}{\M(S) \in \Omega} \leq e^\eps \P{\M}{\M(S') \in \Omega} + \delta.$$ When $\delta = 0$, $\M$ satisfies \emph{pure} differential privacy. When $\delta > 0$, $\M$ satisfies \emph{approximate} differential privacy.
\end{definition}

Differential privacy constrains the output of an algorithm $\M$. However, there is no intermediary between $\M$ and the raw database, so $\M$ must be trusted as a curator of this data. Since $\M$ enjoys this ``central'' access to all data, the above model of differential privacy is sometimes called ``central'' differential privacy.

We focus on the more restrictive \emph{local} model of differential privacy~\cite{DMNS06, BNO08, KLNRS11}. In this model, $\M$ no longer enjoys privileged access to the entire database, and individuals need not trust any central curator. Instead, we view the size-$n$ database as distributed, row by row, among $n$ individuals. Any computation occurs on privatized outputs from each of the $n$ individuals, and the private computation becomes a public interaction between users. Accordingly, we imitate the interactive transcript-based framework of~\citet{JMNR19} and view the computation as an interaction between the $n$ individuals that is coordinated by a protocol $\A$. In each round of this interaction, the protocol $\A$ observes the transcript of interactions so far and selects a set of users and \emph{randomizers}.

\begin{definition}
	An \emph{$(\eps,\delta)$-randomizer} $R \colon \X \to \mathcal{O}$ is an $(\eps,\delta)$-differentially private function taking a single data point as input.
\end{definition}

The users apply the assigned randomizers to their input and publish the output. $\A$ then observes the updated transcript, selects a new users-randomizers pair, and the interaction continues.

\begin{definition}
	A \emph{transcript} $\trar$ is a vector consisting of 5-tuples $(S_U^t, S_R^t, S_\eps^t, S_\delta^t, S_Y^t)$ --- encoding the set of users chosen, set of randomizers assigned, set of randomizer privacy parameters, and set of randomized outputs produced --- for each round $t$. $\trar_{<t}$ denotes the transcript prefix before round $t$. Letting $S_\pi$ denote the collection of all transcripts and $S_R$ the collection of all randomizers, a \emph{protocol} is a function $\A \colon S_\pi \to \left(2^{[n]} \times 2^{S_R} \times 2^{\mathbb{R}_{\geq 0}} \times 2^{\mathbb{R}_{\geq 0}}\right) \cup \{\perp\}$ mapping transcripts to sets of users, randomizers, and randomizer privacy parameters ($\perp$ is a special character indicating a protocol halt). The length of the transcript, as indexed by $t$, is its \emph{round complexity}.
\end{definition}

This framework leads to a simple definition of local differential privacy: the transcript of a locally differentially private interaction is differentially private in the user inputs.

\begin{definition}
\label{def:local_dp}
	Given $\eps, \delta \geq 0$, randomized protocol $\A$ on (distributed) database $S$ satisfies $(\eps,\delta)$-local differential privacy if the transcript it generates is an $(\eps,\delta)$-differentially private function of $S$.
\end{definition}

For brevity, we often shorthand ``differentially private'' as ``private''. In general, we will distinguish between three modes of interactivity in locally private protocols. In \emph{noninteractive} protocols, users output their communications in a single simultaneous round of communication. In \emph{sequentially interactive} protocols, interaction occurs in a sequence: users may base their communication on previous messages, but only speak once. Lastly, communication in \emph{fully interactive protocols} may depend on previous communications, and users may also communicate arbitrarily many times. Illustrations of these models appear in Figure~\ref{fig:interaction_models}.

\vspace{10pt}

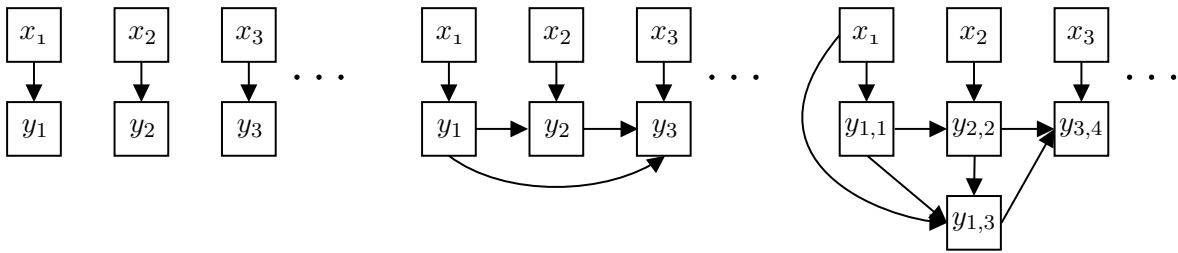
\begin{figure}[h]
\centering

\tikzset{every picture/.style={line width=0.75pt}} %set default line width to 0.75pt        

\begin{tikzpicture}[x=0.75pt,y=0.75pt,yscale=-1,xscale=1]
%uncomment if require: \path (0,221); %set diagram left start at 0, and has height of 221

%Shape: Rectangle [id:dp5644959554599001] 
\draw   (10,30.09) -- (37.09,30.09) -- (37.09,57.18) -- (10,57.18) -- cycle ;
%Shape: Rectangle [id:dp2556100768098617] 
\draw   (64.18,30.09) -- (91.27,30.09) -- (91.27,57.18) -- (64.18,57.18) -- cycle ;
%Shape: Rectangle [id:dp9546429093724473] 
\draw   (118.36,30.09) -- (145.46,30.09) -- (145.46,57.18) -- (118.36,57.18) -- cycle ;
%Straight Lines [id:da2675918273388078] 
\draw    (23.55,57.18) -- (23.55,75.5) ;
\draw [shift={(23.55,77.5)}, rotate = 270] [fill={rgb, 255:red, 0; green, 0; blue, 0 }  ][line width=0.75]  [draw opacity=0] (8.93,-4.29) -- (0,0) -- (8.93,4.29) -- cycle    ;

%Straight Lines [id:da5184087996998956] 
\draw    (77.73,57.18) -- (77.73,75.5) ;
\draw [shift={(77.73,77.5)}, rotate = 270] [fill={rgb, 255:red, 0; green, 0; blue, 0 }  ][line width=0.75]  [draw opacity=0] (8.93,-4.29) -- (0,0) -- (8.93,4.29) -- cycle    ;

%Straight Lines [id:da7792441878891176] 
\draw    (131.91,57.18) -- (131.91,75.5) ;
\draw [shift={(131.91,77.5)}, rotate = 270] [fill={rgb, 255:red, 0; green, 0; blue, 0 }  ][line width=0.75]  [draw opacity=0] (8.93,-4.29) -- (0,0) -- (8.93,4.29) -- cycle    ;

%Shape: Rectangle [id:dp01620760929099574] 
\draw   (10,77.5) -- (37.09,77.5) -- (37.09,104.59) -- (10,104.59) -- cycle ;
%Shape: Rectangle [id:dp3458103667321957] 
\draw   (64.18,77.5) -- (91.27,77.5) -- (91.27,104.59) -- (64.18,104.59) -- cycle ;
%Shape: Rectangle [id:dp6067533232074456] 
\draw   (118.36,77.5) -- (145.46,77.5) -- (145.46,104.59) -- (118.36,104.59) -- cycle ;
%Shape: Rectangle [id:dp9873288184364426] 
\draw   (219.28,30.09) -- (246.37,30.09) -- (246.37,57.18) -- (219.28,57.18) -- cycle ;
%Shape: Rectangle [id:dp5841672230386168] 
\draw   (273.46,30.09) -- (300.55,30.09) -- (300.55,57.18) -- (273.46,57.18) -- cycle ;
%Shape: Rectangle [id:dp2758757098737423] 
\draw   (327.64,30.09) -- (354.73,30.09) -- (354.73,57.18) -- (327.64,57.18) -- cycle ;
%Straight Lines [id:da044939261201732306] 
\draw    (232.82,57.18) -- (232.82,75.5) ;
\draw [shift={(232.82,77.5)}, rotate = 270] [fill={rgb, 255:red, 0; green, 0; blue, 0 }  ][line width=0.75]  [draw opacity=0] (8.93,-4.29) -- (0,0) -- (8.93,4.29) -- cycle    ;

%Straight Lines [id:da620961764786605] 
\draw    (287.01,57.18) -- (287.01,75.5) ;
\draw [shift={(287.01,77.5)}, rotate = 270] [fill={rgb, 255:red, 0; green, 0; blue, 0 }  ][line width=0.75]  [draw opacity=0] (8.93,-4.29) -- (0,0) -- (8.93,4.29) -- cycle    ;

%Straight Lines [id:da5563938743905845] 
\draw    (341.19,57.18) -- (341.19,75.5) ;
\draw [shift={(341.19,77.5)}, rotate = 270] [fill={rgb, 255:red, 0; green, 0; blue, 0 }  ][line width=0.75]  [draw opacity=0] (8.93,-4.29) -- (0,0) -- (8.93,4.29) -- cycle    ;

%Shape: Rectangle [id:dp34191814396666187] 
\draw   (219.28,77.5) -- (246.37,77.5) -- (246.37,104.59) -- (219.28,104.59) -- cycle ;
%Shape: Rectangle [id:dp01132116981486253] 
\draw   (273.46,77.5) -- (300.55,77.5) -- (300.55,104.59) -- (273.46,104.59) -- cycle ;
%Shape: Rectangle [id:dp953757838017401] 
\draw   (327.64,77.5) -- (354.73,77.5) -- (354.73,104.59) -- (327.64,104.59) -- cycle ;
%Straight Lines [id:da22601225287223592] 
\draw    (246.37,91.04) -- (270.78,91.04) ;
\draw [shift={(272.78,91.04)}, rotate = 180] [fill={rgb, 255:red, 0; green, 0; blue, 0 }  ][line width=0.75]  [draw opacity=0] (8.93,-4.29) -- (0,0) -- (8.93,4.29) -- cycle    ;

%Straight Lines [id:da6110228938167666] 
\draw    (300.55,91.04) -- (326.32,91.04) ;
\draw [shift={(328.32,91.04)}, rotate = 180] [fill={rgb, 255:red, 0; green, 0; blue, 0 }  ][line width=0.75]  [draw opacity=0] (8.93,-4.29) -- (0,0) -- (8.93,4.29) -- cycle    ;

%Shape: Rectangle [id:dp28029779618237427] 
\draw   (429.91,30.09) -- (457,30.09) -- (457,57.18) -- (429.91,57.18) -- cycle ;
%Shape: Rectangle [id:dp6212659159679585] 
\draw   (484.09,30.09) -- (511.19,30.09) -- (511.19,57.18) -- (484.09,57.18) -- cycle ;
%Shape: Rectangle [id:dp1960074388807198] 
\draw   (538.28,30.09) -- (565.37,30.09) -- (565.37,57.18) -- (538.28,57.18) -- cycle ;
%Straight Lines [id:da24015749015014176] 
\draw    (443.46,57.18) -- (443.46,75.5) ;
\draw [shift={(443.46,77.5)}, rotate = 270] [fill={rgb, 255:red, 0; green, 0; blue, 0 }  ][line width=0.75]  [draw opacity=0] (8.93,-4.29) -- (0,0) -- (8.93,4.29) -- cycle    ;

%Straight Lines [id:da29904857457046363] 
\draw    (497.64,57.18) -- (497.64,75.5) ;
\draw [shift={(497.64,77.5)}, rotate = 270] [fill={rgb, 255:red, 0; green, 0; blue, 0 }  ][line width=0.75]  [draw opacity=0] (8.93,-4.29) -- (0,0) -- (8.93,4.29) -- cycle    ;

%Straight Lines [id:da2674511920250926] 
\draw    (551.82,57.18) -- (551.82,75.5) ;
\draw [shift={(551.82,77.5)}, rotate = 270] [fill={rgb, 255:red, 0; green, 0; blue, 0 }  ][line width=0.75]  [draw opacity=0] (8.93,-4.29) -- (0,0) -- (8.93,4.29) -- cycle    ;

%Shape: Rectangle [id:dp5193685945351916] 
\draw   (429.91,77.5) -- (457,77.5) -- (457,104.59) -- (429.91,104.59) -- cycle ;
%Shape: Rectangle [id:dp5631243366123306] 
\draw   (484.09,77.5) -- (511.19,77.5) -- (511.19,104.59) -- (484.09,104.59) -- cycle ;
%Shape: Rectangle [id:dp12968316205545416] 
\draw   (538.28,77.5) -- (565.37,77.5) -- (565.37,104.59) -- (538.28,104.59) -- cycle ;
%Straight Lines [id:da3039528123254711] 
\draw    (457.68,91.04) -- (482.09,91.04) ;
\draw [shift={(484.09,91.04)}, rotate = 180] [fill={rgb, 255:red, 0; green, 0; blue, 0 }  ][line width=0.75]  [draw opacity=0] (8.93,-4.29) -- (0,0) -- (8.93,4.29) -- cycle    ;

%Curve Lines [id:da5865097848562033] 
\draw    (232.82,104.59) .. controls (260.84,125.94) and (313.18,124.94) .. (340.63,105.49) ;
\draw [shift={(341.87,104.59)}, rotate = 503.13] [fill={rgb, 255:red, 0; green, 0; blue, 0 }  ][line width=0.75]  [draw opacity=0] (8.93,-4.29) -- (0,0) -- (8.93,4.29) -- cycle    ;

%Shape: Rectangle [id:dp5946667110793777] 
\draw   (484.09,124.91) -- (511.19,124.91) -- (511.19,152) -- (484.09,152) -- cycle ;
%Straight Lines [id:da4899568947222037] 
\draw    (497.98,104.59) -- (497.67,122.91) ;
\draw [shift={(497.64,124.91)}, rotate = 270.95] [fill={rgb, 255:red, 0; green, 0; blue, 0 }  ][line width=0.75]  [draw opacity=0] (8.93,-4.29) -- (0,0) -- (8.93,4.29) -- cycle    ;

%Curve Lines [id:da8672164778211133] 
\draw    (429.91,43.64) .. controls (376.15,107.35) and (451.18,136.6) .. (482.25,138.38) ;
\draw [shift={(484.09,138.45)}, rotate = 181.3] [fill={rgb, 255:red, 0; green, 0; blue, 0 }  ][line width=0.75]  [draw opacity=0] (8.93,-4.29) -- (0,0) -- (8.93,4.29) -- cycle    ;

%Straight Lines [id:da3157304145616182] 
\draw    (443.8,104.59) -- (482.56,137.17) ;
\draw [shift={(484.09,138.45)}, rotate = 220.04] [fill={rgb, 255:red, 0; green, 0; blue, 0 }  ][line width=0.75]  [draw opacity=0] (8.93,-4.29) -- (0,0) -- (8.93,4.29) -- cycle    ;

%Straight Lines [id:da04659867083128266] 
\draw    (511.52,91.04) -- (536.28,91.04) ;
\draw [shift={(538.28,91.04)}, rotate = 180] [fill={rgb, 255:red, 0; green, 0; blue, 0 }  ][line width=0.75]  [draw opacity=0] (8.93,-4.29) -- (0,0) -- (8.93,4.29) -- cycle    ;

%Straight Lines [id:da1153192032508561] 
\draw    (511.52,138.45) -- (537.29,92.79) ;
\draw [shift={(538.28,91.04)}, rotate = 479.44] [fill={rgb, 255:red, 0; green, 0; blue, 0 }  ][line width=0.75]  [draw opacity=0] (8.93,-4.29) -- (0,0) -- (8.93,4.29) -- cycle    ;

% Text Node
\draw (23.55,43.64) node   {$\mathnormal{x_{1}}$};
% Text Node
\draw (78.41,43.64) node   {$x_{2}$};
% Text Node
\draw (132.59,43.64) node   {$x_{3}$};
% Text Node
\draw (169.5,65.65) node [scale=1.7280000000000002]  {$\cdots $};
% Text Node
\draw (24.22,91.04) node   {$y_{1}$};
% Text Node
\draw (132.59,91.04) node   {$y_{3}$};
% Text Node
\draw (78.41,91.04) node   {$y_{2}$};
% Text Node
\draw (233.5,43.64) node   {$\mathnormal{x_{1}}$};
% Text Node
\draw (287.68,43.64) node   {$x_{2}$};
% Text Node
\draw (341.87,43.64) node   {$x_{3}$};
% Text Node
\draw (378.78,65.65) node [scale=1.7280000000000002]  {$\cdots $};
% Text Node
\draw (233.5,91.04) node   {$y_{1}$};
% Text Node
\draw (341.87,91.04) node   {$y_{3}$};
% Text Node
\draw (287.68,91.04) node   {$y_{2}$};
% Text Node
\draw (443.46,43.64) node   {$\mathnormal{x_{1}}$};
% Text Node
\draw (497.64,43.64) node   {$x_{2}$};
% Text Node
\draw (551.82,43.64) node   {$x_{3}$};
% Text Node
\draw (589.41,65.65) node [scale=1.7280000000000002]  {$\cdots $};
% Text Node
\draw (443.46,91.04) node   {$y_{1,1}$};
% Text Node
\draw (551.82,91.04) node   {$y_{3,4}$};
% Text Node
\draw (497.64,91.04) node   {$y_{2,2}$};
% Text Node
\draw (497.64,138.45) node   {$y_{1,3}$};

\end{tikzpicture}
\caption{From left to right: examples of noninteractive, sequential, and full interaction. In each illustration, $x$ variables are user data, and $y$ variables are privatized user responses. In the noninteractive model, each privatized response $y_i$ is a function only of the user's data $x_i$ (and their internal randomness). In the sequential model, each $y_i$ is a function of $x_i$ and previous responses $y_1, \ldots, y_{i-1}$. In the full model, each $y_{i,t}$ is a function of $x_i$ and any $y_{i',t'}$ for any $t' < t$.}
\label{fig:interaction_models}
\end{figure}

For brevity, we sometimes omit the phrase ``locally private'' when discussing these models. In these cases, the privacy of the protocol in question should be clear from context.

\subsection{Multi-party Setting}
\label{subsec:multi}
We combine the communication model of Section~\ref{subsec:two_party} and the notion of local privacy defined in Section~\ref{subsec:dp} in a \emph{multi-party communication model}.

\begin{definition}
	In the \emph{multi-party communication model} there exists an ``Alice input'' $x \in \X$ and a ``Bob input'' $y \in \Y$. Each of an unboundedly large number of players receives an independent and uniformly random draw over $\{x,y\}$. The users' goal is to output $z \in \Z$ such that $(x,y,z) \in \R$.
\end{definition}

A \emph{protocol} is defined as in the local privacy setting (and we will be interested in local privacy as the primary constraint on protocols). As before, we will quantify the likelihood that a protocol achieves a ``good'' outcome. Unlike before, our metric of interest for a multi-party protocol is its \emph{sample complexity}.

\begin{definition}
	The \emph{sample complexity} of a protocol $\A$, denoted by $\sac(\A)$, is the maximum number of users appearing in the transcript over all inputs $(x,y)$ and all random coins. If on all inputs $x$ and $y$ $\P{\A}{(x,y,z) \not \in \mathcal{R}} \leq \gamma$, we say $\A$ \emph{computes $\R$ with error at most $\gamma$}. The \emph{randomized sample complexity with error $\gamma$} of a relation $\R$ is then defined as $$\sac_{\gamma} (\R) = \min_{\A:\A \text{ computes } \R \text{ with error at most } \gamma} \sac(\A).$$ Let $\sac_{\gamma}^{\eps, \non}(\R), \sac_{\gamma}^{\eps, \seq}(\R)$, and $\sac_{\gamma}^{\eps, \full}(\R)$ denote the sample complexities of $\eps$-locally private protocols computing $\R$ with error $\gamma$ under noninteraction, sequential interaction, and full interaction respectively. 
\end{definition}

Our two-party models are defined by an input pair $(x,y)$, and we deliberately constrain our multi-party models to be defined by a pair $(x,y)$ as well. Thus a given two-party problem on a pair of inputs induces a unique multi-party problem on the same inputs, and vice-versa. We note that multi-party problems as we define them are unusual statistical estimation problems: their primary use for us will be in proving lower bounds.  

\section{First Reduction: An Equivalence}
\label{sec:eq_red}
We now prove an equivalence between a two-party problem's communication complexity and the induced multi-party problem's  sample complexity for private sequentially interactive protocols. To do so, we first show that the equivalence holds when the two-party communication channel is noisy in a specific way (Section~\ref{subsec:noisy}). Pairing this with existing results relating communication complexity over noisy and non-noisy channels (Section~\ref{subsec:noisy_2}) completes the result (Section~\ref{subsec:eq}).

\subsection{Noisy Two-party Communication}
\label{subsec:noisy}
In the noisy two-party communication model, Alice and Bob may only communicate over a \emph{binary symmetric channel}.

\begin{definition}
	For $\eps \in (0,1/2)$, a \emph{binary symmetric channel with crossover probability $\eps$} $\bsc_{\eps}$ correctly transmits a bit $b$ with probability $1/2 + \eps$ and transmits $1-b$ with probability $1/2 - \eps$. We additionally suppose that the binary symmetric channel has \emph{feedback}: the sender always sees the received bit.
\end{definition}

Let $\cc_{\gamma}^{\eps}(\R)$ denote the communication complexity of $\R$ with error $\gamma$ under the additional requirement that communication occurs over $\bsc_{\eps}$. First, we show how to transform two-party protocols over a binary symmetric channel with crossover probability depending on $\eps$ into multi-party protocols.

\begin{lemma}
\label{lem:two_to_m}
	Let $\R$ be a relation for some communication problem, $\eps \geq 0$, $\eps' =  \tfrac{e^\eps-1}{4(e^\eps+1)}$, and $\gamma \in (0,1)$. Then $\sac_{\gamma}^{\eps, \seq}(\R) = O\left(\cc_{\gamma}^{\eps'}(\R)\right)$.
\end{lemma}
\begin{proof}
	Let $\A_2$ be any protocol for the two-party problem over $\bsc_{\eps'}$ computing $\R$ with error $\gamma$.  Consider the first bit sent in $\A_2$. Without loss of generality, Alice sends this first bit $f(x)$, where $x$ is Alice's input. Since communication occurs over $\bsc_{\eps}$, with probability $1/2 + \eps'$ Bob receives $f(x)$, and with probability $1/2 - \eps'$ Bob receives its negation.
		
We will use $\A_2$ to build a multi-party protocol $\A_m$. To simulate this bit, $\A_m$ selects a new (previously un-selected) agent, and the new agent takes one of two actions. If the agent has an Alice input $x$, they apply randomized response to $f(x)$: i.e. they send $f(x)$ with probability $\tfrac{e^\eps}{e^\eps + 1}$ and otherwise send $(1-f(x))$. If instead the agent has a Bob input $y$, they send a uniform random bit. Thus the probability that the agent sends $f(x)$ is
\begin{align*}
	\P{}{\text{Alice input}} \cdot \tfrac{e^\eps}{e^\eps+1} + \P{}{\text{Bob input}} \cdot \tfrac{1}{2} =&\; \tfrac{e^\eps}{2(e^\eps+1)} + \tfrac{1}{4} \\
	=&\; \tfrac{2e^\eps}{4(e^\eps+1)} + \tfrac{e^\eps+1}{4(e^\eps+1)} \\
	=&\; \tfrac{3e^\eps+1}{4(e^\eps+1)} \\
	=&\; \tfrac{2(e^\eps+1)}{4(e^\eps+1)} + \tfrac{e^\eps-1}{4(e^\eps+1)} \\
	=&\; \tfrac{1}{2} + \eps'.
\end{align*}

 It follows that the first bit of $\A_m$ is distributed identically to the first bit of $\A_2$. Repeating this process for each bit sent in $\A_2$, $\A_m$ induces an identical distribution over the bits output, and thus computes $\R$ with error $\gamma$. Since each bit sent in $\A_2$ used a new user in $\A_m$, $\sac(\A_m)= O\left(\cc^{\eps'}(\A_2)\right)$. Since randomized response satisfies $\epsilon$-differential privacy, the sequentially interactive mechanism $\A_m$ is $\epsilon$-differentially private in the local model.
\end{proof}

Next, we show how to transform multi-party protocols into two-party protocols over a binary symmetric channel with crossover probability depending (in a slightly different way) on $\eps$.

\begin{lemma}
\label{lem:m_to_two}
		Let $\R$ be a relation for some communication problem, $0 < \eps = O(1)$, $\eps' =  \tfrac{e^\eps-1}{2(e^\eps+1)}$, and $\gamma, \eta > 0$ such that $\gamma + \eta < 1$. Then $\cc_{\gamma + \eta}^{\eps'}(\R) = O\left(\tfrac{1}{\eta} \cdot \sac_{\gamma}^{\eps, \seq}(\R)\right)$.
\end{lemma}
\begin{proof}
	Here, we define $\eps' = \tfrac{e^\eps-1}{2(e^\eps+1)}$, which differs from our previous $\eps'$ by a factor of 2. Let $\A_m$ be any sequentially interactive $\eps$-locally private protocol for a multi-party problem computing $\R$ with error $\gamma$ and sample complexity $n$. By the following result from Bassily and Smith~\cite{BS15}, we can transform $\A_m$ into a new, functionally equivalent protocol $\A_m'$ in which each user sends only one bit.
		
		\begin{lemma}[Theorem 4.1 in~\citet{BS15}]
		\label{lem:one_bit}
		Given an $\eps$-locally private protocol $\A$ with expected number of randomizer calls $T$, there exists a sequentially interactive $(\eps,0)$-locally private protocol $\A'$ with expected number of users $e^\eps \cdot T$ where each user sends a single bit (produced by a call to a single $\eps$-local randomizer). Moreover, there exists a deterministic function $f$ on transcripts such that $f(\Pi(\A')) = \Pi(\A)$, where $\Pi(\cdot)$ denotes a distribution over transcripts induced  by a given protocol with randomness is over the protocol and its samples.
		\end{lemma}
		
		The cost is twofold. First, $\A_m'$ requires $O(n \log(\log(n)))$ bits of public randomness. Second, $\A_m'$ requires $e^\eps n$ users in expectation. By Markov's inequality, the number of users can be bounded by $\tfrac{e^\eps n}{\eta}$ at the cost of an $\eta$ increase in failure probability.
		
		We now transform $\A_m'$ into a two-player protocol $\A_2$. The idea will be to have Alice and Bob simulate $A_m'$ by randomly partitioning the users from the multi-party protocol between themselves, and each simulating the role of their assigned users. Recall that for multi-party communication problems, users are randomly assigned ``Alice'' or ``Bob'' data points, and so this random partition will induce the correct distribution on data elements. Thus, $\A_2$ begins with Alice and Bob using their shared public randomness to generate $\tfrac{e^\eps n}{\eta}$  coin flips determining who will simulate which agents.
		
		Without loss of generality, suppose Alice simulates the first agent. Let $$p_x = \P{}{\text{agent sends } 1 \mid \text{agent has Alice's data } x},$$ and let $p_{\min} = \min_{x \in \X} p_x$ and $p_{\max} = \max_{x \in \X} p_x$.  Alice and Bob take one of two choices depending on $p_{\min}$ and $p_{\max}$.
		
		\underline{Case 1}: $p_{\min} + p_{\max} \leq 1$. Then Alice sends 1 with probability $\tfrac{1}{2} + \tfrac{p_x}{2\eps'(p_{\min} + p_{\max})} - \tfrac{1}{4\eps'}$ and sends 0 with the remaining probability. Since
		\begin{align*}
			\tfrac{1}{2} + \tfrac{p_x}{2\eps'(p_{\min} + p_{\max})} - \tfrac{1}{4\eps'} =&\; \tfrac{1}{2} + \tfrac{2p_x - p_{\min} - p_{\max}}{4\eps'(p_{\min} + p_{\max})} \\
			\leq&\; \tfrac{1}{2} + \tfrac{p_{\max} - p_{\min}}{4\eps'(p_{\min} + p_{\max})} \\
			=&\; \tfrac{1}{2} + \tfrac{e^\eps+1}{2(e^\eps-1)} \cdot \tfrac{p_{\max} - p_{\min}}{p_{\max} + p_{\min}} \\
			=&\; \tfrac{1}{2} + \tfrac{e^\eps+1}{2(e^\eps-1)} \cdot \left[1 - \tfrac{2p_{\min}}{p_{\max} + p_{\min}}\right] \\
			\leq&\; \tfrac{1}{2} + \tfrac{e^\eps+1}{2(e^\eps-1)}\left[1 - \tfrac{2}{e^\eps+1}\right] = 1
		\end{align*}
(where both inequalities use the fact that the agent sends output from an $\eps$-local randomizer), and similarly
	\begin{align*}
		\tfrac{1}{2} + \tfrac{2p_x - p_{\min} - p_{\max}}{4\eps'(p_{\min} + p_{\max})} \geq&\; \tfrac{1}{2} + \tfrac{e^\eps+1}{2(e^\eps-1)} \cdot \tfrac{p_{\min} - p_{\max}}{p_{\max} + p_{\min}} \\
		\geq&\; \tfrac{1}{2} + \tfrac{e^\eps+1}{2(e^\eps-1)}\left[\tfrac{2}{e^\eps+1} - 1\right] = 0
	\end{align*}
these are valid probabilities. Next, as Alice sends the bit over $\bsc_{\eps'}$, the probability that the received bit is 1 is $$\left(\tfrac{1}{2} + \eps'\right)\left(\tfrac{1}{2} + \tfrac{p_x}{2\eps'(p_{\min} + p_{\max})} - \tfrac{1}{4\eps'}\right) + \left(\tfrac{1}{2} - \eps'\right)\left(\tfrac{1}{2} - \tfrac{p_x}{2\eps'(p_{\min} + p_{\max})} + \tfrac{1}{4\eps'}\right) = \tfrac{p_x}{p_{\min} + p_{\max}}.$$ With probability $p_{\min} + p_{\max}$, Alice and Bob ``use'' the received bit: that is, they enter this received bit into their transcript and continue the protocol. With probability $1 - p_{\min} - p_{\max}$ Alice and Bob instead enter the bit 0 into their transcript (and omit the true received bit from the transcript) and continue. Then $\P{}{\text{enter 1 in transcript}}$ (and $\P{}{\text{enter 0 in transcript}}$) are identical in both the two-party and multi-party protocols.

		\underline{Case 2}: $p_{\min} + p_{\max} > 1$. Then if we define $p_{\min}'$ and $p_{\max}'$ as $1 - p_{\min}$ and $1 - p_{\max}$ respectively, we get $p_{\min}' + p_{\max}' < 1$. Let $p_x' = 1 - p_x$ and have Alice send 0 with probability $\tfrac{1}{2} + \tfrac{p_x'}{2\eps'(p_{\min}' + p_{\max}')} - \tfrac{1}{4\eps'}$, and Alice and Bob ``use'' the received bit (just as defined in Case 1) with probability $p_{\min}' + p_{\max}'$. Repeating the analysis from Case 1 for $p_x'$, $p_{\min}'$, and $p_{\max}'$ yields that $\P{}{\text{use 0}}$ (and $\P{}{\text{use 1}}$) are identical in both the two-party and multi-party protocols.
		
		Combining Cases 1 and 2, Alice and Bob produce the same distribution over the first bit of the protocol in $\A_2$ and $\A_m$. Since we can repeat this process for subsequent bits, by induction the distribution over transcripts (and thus answers) is identical. Therefore $\A_2$  also computes $\R$ with error $\gamma$. Moreover, there is a one-to-one correspondence between users in $\A_m$ and bits in $\A_2$, so by $\eps = O(1)$, $\cc_{\gamma + \eta}^{\eps'}(\mathcal{R}) = O\left(\tfrac{1}{\eta} \cdot \sac_{\gamma}^{\eps,S}(\mathcal{R})\right)$.
\end{proof}

\subsection{Relating Noisy and Noiseless Communication}
\label{subsec:noisy_2}
Lemmas~\ref{lem:two_to_m} and~\ref{lem:m_to_two} relate the sample complexity of sequentially interactive $\eps$-locally private multi-party protocols and the communication complexity of two-party protocols over a binary symmetric channel. The remaining step is to relate noisy and noiseless communication complexity. Fortunately, one direction of this relationship follows almost immediately from previous work by Braverman and Mao~\cite{BM15}.

\begin{lemma}[Theorem 3.1 in Braverman and Mao~\cite{BM15}]
\label{lem:bm15}
	For every protocol $\A$ over $\bsc_{\eps}$ with feedback, there exists a protocol $\A'$ over a noiseless channel that simulates $\A$ with $\overbar{\cc}(\A') = O\left(\eps^2 \cc^{\eps}(\A)\right)$. Here, $\overbar{\cc}\left(\A'\right)$ is the maximum over all inputs $(x,y)$ of the expected number of bits exchanged over the randomness of $\A'$.
\end{lemma}

By Markov's inequality, we get a high-probability version of their result for our setting.

\begin{lemma}
\label{lem:bm15_good}
	Let $\R$ be a relation for a two-party communication problem. Then for $\gamma, \eta >0$ where $\gamma + \eta <1$, $\cc_{\gamma + \eta}\left(\R\right) = O\left(\tfrac{\eps^2}{\eta} \cdot \cc_{\gamma}^{\eps}(\R)\right)$.
\end{lemma}

It remains to upper bound noisy communication complexity using noiseless communication complexity. Schulman~\cite{S96} first studied the problem of interactive two-party communication through noisy channels and showed how to simulate a noiseless channel using a binary symmetric channel with small ($<1/240$) crossover probability and a constant blowup in communication complexity. Braverman and Rao~\cite{BR14} then improved this result to binary symmetric channels with crossover probability bounded away from $1/8$ by a constant.

\begin{lemma}[Simplified Version of Theorem 2 in Braverman and Rao~\cite{BR14}]
\label{lem:br14}
	Let $\R$ be a relation for a two-party communication problem, $\gamma \in (0,1)$, and $0 < p \leq 1/8 - c$ where $c = \Omega(1)$. Then $\cc_{\gamma}^p(\R) = O\left(\cc_{\gamma}(\R)\right)$.
\end{lemma}

One obstacle remains: Lemma~\ref{lem:br14} requires a channel $C$ with crossover probability bounded away from $1/8$, while our channel $C'$ may have crossover probability $\eps$-close to $1/2$. We therefore use a standard amplification argument, replacing each bit over $C$ with $\Theta(1/\eps^2)$ bits over $C'$ and taking the majority as the transmitted bit. This yields Lemma~\ref{lem:br14_good}.

\begin{lemma}
\label{lem:br14_good}
	Let $\R$ be a relation for a two-party communication problem. Then $\cc_{\gamma}^{\eps}(\R) = O\left(\tfrac{1}{\eps^2} \cdot \cc_{\gamma}(\R)\right)$.
\end{lemma}

\subsection{Equivalence}
\label{subsec:eq}
The results above come together in the following equivalence.

\begin{theorem}
\label{thm:eq}
	Let $\R$ be a relation for some communication problem. Then for any $\eps = O(1)$ and $0 < \gamma, \eta = \Omega(1)$ such that $\gamma + \eta < 1$, $\sac_\gamma^{\eps,S}(\R) = \Theta\left(\tfrac{1}{\eps^2} \cdot \cc_{\gamma + \eta}(\R)\right)$.
\end{theorem}
\begin{proof}
	We first show $\sac_{\gamma}^{\eps,S}(\R) = O\left(\tfrac{1}{\eps^2} \cdot \cc_\gamma(\R)\right)$. By Lemma~\ref{lem:two_to_m}, $\sac_{\gamma}^{\eps,S}(\R) = O\left(\cc_{\gamma}^{\eps_1}(\R)\right)$ where $\eps_1 =  \tfrac{e^\eps-1}{4(e^\eps+1)}$. Then, by Lemma~\ref{lem:br14_good}, $\cc_{\gamma}^{\eps_1}(\R) = O\left(\tfrac{1}{\eps_1^2} \cdot \cc_{\gamma}(\R)\right)$. Since $\eps = O(1)$, $\eps_1 = \Omega(\eps)$, and tracing back yields the claim.
	
	Next, we show $\cc_{\gamma + \eta}(\R) = O\left(\eps^2 \cdot \sac_{\gamma}^{\eps, S}(\R)\right)$. Since $\gamma, \eta = \Omega(1)$, by Lemma~\ref{lem:bm15_good} $\cc_{\gamma + \eta}(\R) = O\left(\eps_1^2 \cdot \cc_{\gamma + \eta/2}^{\eps_1}(\R)\right)$ where $\eps_1 =  \tfrac{e^\eps-1}{2(e^\eps+1)}$. By Lemma~\ref{lem:m_to_two}, $\cc_{\gamma + \eta/2}^{\eps_1}(\R) = O\left(\sac_{\gamma}^{\eps, S}(\R)\right)$. Tracing back and using $\eps_1 = O(\eps)$ implies the claim.
\end{proof}

\section{Separating Sequential and Full Interactivity}
\label{sec:si_fi}
In this section, we use Theorem~\ref{thm:eq} to show that the \emph{hidden layers problem} is hard for sequentially interactive protocols (Corollary~\ref{cor:lb}). In contrast, we show that the same problem is ``easy'' for fully interactive protocols (Theorem~\ref{thm:fi_ub}). We start by introducing the problem below.

\subsection{Hidden Layers Problem $\hl$}
\label{subsec:hlp}
In this section, we formally recap the hidden layers problem that drives our results. While~\citet{B13} first proposed this problem, we imitate the presentation of~\citet{GKR16}.

The hidden layers problem is parameterized by $k \in \mathbb{N}$ and denoted $\hl(k)$. It features a $2^{4k}$-ary tree $\T$ with directed edges from root to leaves and $2^{rs} + 1$ layers where $r = 2^{2^{8k}}$ and $s = 2^{8k}$. $\T$ thus has a number of layers triply exponential in $k$ and a number of leaves quadruply exponential in $k$. Two players, Alice and Bob, each receive a small amount of information about $\T$. Alice receives $(a,f)$ where $a \in \{0,2, \ldots, 2^{rs}-2\}$ indexes an even-numbered layer, and $f$ labels each vertex in layer $a$ of $\T$ with a single outgoing edge. Similarly, Bob receives $(b,g)$ where $b \in \{1, 3, \ldots, 2^{rs}-1\}$ indexes an odd-numbered layer, and $g$ labels each vertex in layer $b$ with a single outgoing edge. Thus, Alice and Bob each have information about one ``hidden layer'' of $\T$. Letting $v$ be a leaf of $\T$, we say $v$ is \emph{consistent} with $(a,f)$ (or $(b,g)$)  if the path from the root to $v$ goes through an edge identified by $f$ (or $g$, respectively).

The hidden layers problem is a search problem: Alice and Bob must output the same leaf $v$ consistent with both $(a,f)$ and $(b,g)$. Crucially, many such $v$ exist. If at the end of protocol $\A$ Alice and Bob output different leaves, or at least one of them outputs a leaf not consistent with at least one of $(a,f)$ and $(b,g)$, we say $\A$ \emph{errs}. A simplified illustration of the hidden layers problem appears in Figure~\ref{fig:hl}.

\vspace{10pt}

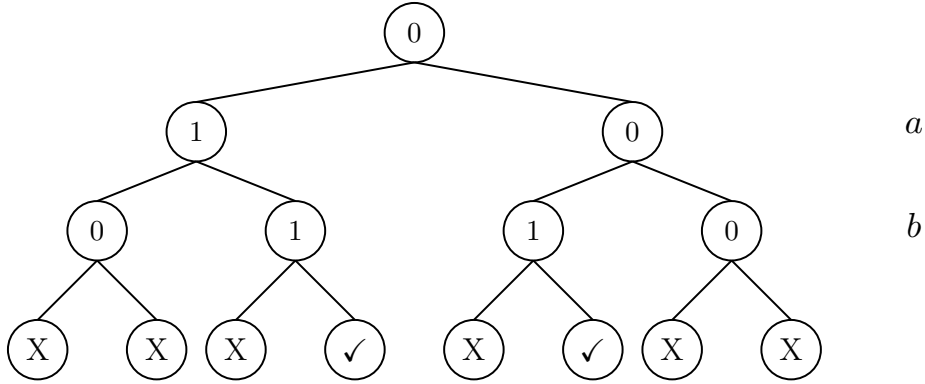
\begin{figure}[h]
\centering

\tikzset{every picture/.style={line width=0.75pt}} %set default line width to 0.75pt

\begin{tikzpicture}[x=0.75pt,y=0.75pt,yscale=-1,xscale=1]
%uncomment if require: \path (0,212); %set diagram left start at 0, and has height of 212

%Shape: Circle [id:dp8682648485313968]
\draw   (220,25) .. controls (220,16.72) and (226.72,10) .. (235,10) .. controls (243.28,10) and (250,16.72) .. (250,25) .. controls (250,33.28) and (243.28,40) .. (235,40) .. controls (226.72,40) and (220,33.28) .. (220,25) -- cycle ;
%Shape: Circle [id:dp17336356419423526]
\draw   (110,75) .. controls (110,66.72) and (116.72,60) .. (125,60) .. controls (133.28,60) and (140,66.72) .. (140,75) .. controls (140,83.28) and (133.28,90) .. (125,90) .. controls (116.72,90) and (110,83.28) .. (110,75) -- cycle ;
%Shape: Circle [id:dp09780014561759143]
\draw   (160,125) .. controls (160,116.72) and (166.72,110) .. (175,110) .. controls (183.28,110) and (190,116.72) .. (190,125) .. controls (190,133.28) and (183.28,140) .. (175,140) .. controls (166.72,140) and (160,133.28) .. (160,125) -- cycle ;
%Shape: Circle [id:dp44717562211093]
\draw   (60,125) .. controls (60,116.72) and (66.72,110) .. (75,110) .. controls (83.28,110) and (90,116.72) .. (90,125) .. controls (90,133.28) and (83.28,140) .. (75,140) .. controls (66.72,140) and (60,133.28) .. (60,125) -- cycle ;
%Shape: Circle [id:dp5185624341112212]
\draw   (90,185) .. controls (90,176.72) and (96.72,170) .. (105,170) .. controls (113.28,170) and (120,176.72) .. (120,185) .. controls (120,193.28) and (113.28,200) .. (105,200) .. controls (96.72,200) and (90,193.28) .. (90,185) -- cycle ;
%Shape: Circle [id:dp9503549043671979]
\draw   (30,185) .. controls (30,176.72) and (36.72,170) .. (45,170) .. controls (53.28,170) and (60,176.72) .. (60,185) .. controls (60,193.28) and (53.28,200) .. (45,200) .. controls (36.72,200) and (30,193.28) .. (30,185) -- cycle ;
%Shape: Circle [id:dp3881995269445042]
\draw   (190,185) .. controls (190,176.72) and (196.72,170) .. (205,170) .. controls (213.28,170) and (220,176.72) .. (220,185) .. controls (220,193.28) and (213.28,200) .. (205,200) .. controls (196.72,200) and (190,193.28) .. (190,185) -- cycle ;
%Shape: Circle [id:dp32298864682976225]
\draw   (130,185) .. controls (130,176.72) and (136.72,170) .. (145,170) .. controls (153.28,170) and (160,176.72) .. (160,185) .. controls (160,193.28) and (153.28,200) .. (145,200) .. controls (136.72,200) and (130,193.28) .. (130,185) -- cycle ;
%Shape: Circle [id:dp3905156468287343]
\draw   (330,75) .. controls (330,66.72) and (336.72,60) .. (345,60) .. controls (353.28,60) and (360,66.72) .. (360,75) .. controls (360,83.28) and (353.28,90) .. (345,90) .. controls (336.72,90) and (330,83.28) .. (330,75) -- cycle ;
%Shape: Circle [id:dp45702960803885917]
\draw   (380,125) .. controls (380,116.72) and (386.72,110) .. (395,110) .. controls (403.28,110) and (410,116.72) .. (410,125) .. controls (410,133.28) and (403.28,140) .. (395,140) .. controls (386.72,140) and (380,133.28) .. (380,125) -- cycle ;
%Shape: Circle [id:dp4112805716260104]
\draw   (280,125) .. controls (280,116.72) and (286.72,110) .. (295,110) .. controls (303.28,110) and (310,116.72) .. (310,125) .. controls (310,133.28) and (303.28,140) .. (295,140) .. controls (286.72,140) and (280,133.28) .. (280,125) -- cycle ;
%Shape: Circle [id:dp7426170872829732]
\draw   (310,185) .. controls (310,176.72) and (316.72,170) .. (325,170) .. controls (333.28,170) and (340,176.72) .. (340,185) .. controls (340,193.28) and (333.28,200) .. (325,200) .. controls (316.72,200) and (310,193.28) .. (310,185) -- cycle ;
%Shape: Circle [id:dp7894457715101169]
\draw   (250,185) .. controls (250,176.72) and (256.72,170) .. (265,170) .. controls (273.28,170) and (280,176.72) .. (280,185) .. controls (280,193.28) and (273.28,200) .. (265,200) .. controls (256.72,200) and (250,193.28) .. (250,185) -- cycle ;
%Shape: Circle [id:dp05749016753304459]
\draw   (410,185) .. controls (410,176.72) and (416.72,170) .. (425,170) .. controls (433.28,170) and (440,176.72) .. (440,185) .. controls (440,193.28) and (433.28,200) .. (425,200) .. controls (416.72,200) and (410,193.28) .. (410,185) -- cycle ;
%Shape: Circle [id:dp9961061755566398]
\draw   (350,185) .. controls (350,176.72) and (356.72,170) .. (365,170) .. controls (373.28,170) and (380,176.72) .. (380,185) .. controls (380,193.28) and (373.28,200) .. (365,200) .. controls (356.72,200) and (350,193.28) .. (350,185) -- cycle ;
%Straight Lines [id:da5871306977604007]
\draw    (235,40) -- (345,60) ;

%Straight Lines [id:da42667078514961987]
\draw    (125,60) -- (235,40) ;

%Straight Lines [id:da909619202784701]
\draw    (125,90) -- (175,110) ;

%Straight Lines [id:da2287379566165848]
\draw    (125,90) -- (75,110) ;

%Straight Lines [id:da14660004340263888]
\draw    (175,140) -- (205,170) ;

%Straight Lines [id:da8018759664035495]
\draw    (175,140) -- (145,170) ;

%Straight Lines [id:da17285861200821495]
\draw    (75,140) -- (105,170) ;

%Straight Lines [id:da872399095834282]
\draw    (75,140) -- (45,170) ;

%Straight Lines [id:da18674082371421252]
\draw    (345,90) -- (395,110) ;

%Straight Lines [id:da37379676590955313]
\draw    (295,110) -- (345,90) ;

%Straight Lines [id:da7356067113004834]
\draw    (395,140) -- (425,170) ;

%Straight Lines [id:da9968597426140235]
\draw    (365,170) -- (395,140) ;

%Straight Lines [id:da25968378813006243]
\draw    (295,140) -- (325,170) ;

%Straight Lines [id:da05309052253814284]
\draw    (265,170) -- (295,140) ;

% Text Node
\draw (235,25) node   {$0$};
% Text Node
\draw (345,75) node   {$0$};
% Text Node
\draw (125,75) node   {$1$};
% Text Node
\draw (175,125) node   {$1$};
% Text Node
\draw (295,125) node   {$1$};
% Text Node
\draw (395,125) node   {$0$};
% Text Node
\draw (75,125) node   {$0$};
% Text Node
\draw (487,72) node [scale=1.2]  {$a$};
% Text Node
\draw (487,122) node [scale=1.2]  {$b$};
% Text Node
\draw (45,185) node [color={rgb, 255:red, 0; green, 0; blue, 0 }  ,opacity=1 ] [align=left] {\textcolor[rgb]{0,0,0}{{\large X}}};
% Text Node
\draw (105,185) node  [align=left] {\textcolor[rgb]{0,0,0}{{\large X}}};
% Text Node
\draw (145,185) node  [align=left] {\textcolor[rgb]{0,0,0}{{\large X}}};
% Text Node
\draw (265,185) node  [align=left] {\textcolor[rgb]{0,0,0}{{\large X}}};
% Text Node
\draw (365,185) node  [align=left] {\textcolor[rgb]{0,0,0}{{\large X}}};
% Text Node
\draw (425,185) node  [align=left] {\textcolor[rgb]{0,0,0}{{\large X}}};
% Text Node
\draw (205,185) node  [align=left] {{\large \textcolor[rgb]{0,0,0}{\checkmark}}};
% Text Node
\draw (325,185) node  [align=left] {{\large \textcolor[rgb]{0,0,0}{\checkmark}}};

\end{tikzpicture}

\caption{A simplified instance of the hidden layers problem. Each node is labeled 0 (left) or 1 (right). For layers $a$ and $b$, these labels correspond to the correct child node. Leaves 4 and 6 are thus the only two leaves consistent with the hidden layers $a$ and $b$. Note that a true instance of the hidden layers problem is much larger.}
\label{fig:hl}
\end{figure}

\citet{GKR16} proved that the hidden layers problem has high communication complexity. To do so, they constructed a specific distribution $P$ over user inputs for their result. When we want to specify the distribution $P$ over user inputs, we write $\hl(k, P)$.

\begin{lemma}[Theorem 1 in~\citet{GKR16}]
\label{lem:hlp_lb}
 There exists constant $k'$ and input distribution $P$ for $((a,f), (b,g))$ such that, for every $k \geq k'$ and protocol $\A$ with $\cc(\A) \leq 2^k$, $\P{P}{\A \text{ errs on } \hl(k,P)} \geq 1 - 2^{-k}$.
\end{lemma}

In particular, Lemma~\ref{lem:hlp_lb} implies that $\Omega(2^k)$ communication is necessary to achieve constant success probability for $\hl(k, P)$. Combining Theorem~\ref{thm:eq} and Lemma~\ref{lem:hlp_lb} gives the following corollary.

\begin{corollary}
\label{cor:lb}
	For $\eps = O(1)$ and $\gamma = \Omega(1)$, $\sac_{\gamma}^{\eps,S}\left(\hl(k, P)\right) = \Omega\left(\tfrac{2^k}{\eps^2}\right)$.
\end{corollary}

\subsection{Fully Interactive Upper Bound for $\hl$}
\label{subsec:si_fi_ub}
In this section we provide a fully interactive protocol \fiub~that solves $\hl(k)$ with a much better sample complexity. \fiub~works by greedily following a path from the root to a leaf, querying users to guide its path as it descends the tree.

Concretely, starting from the root, at each vertex $v$ encountered, for all $2^{4k}$ children $v_j$ of $v$, the analyst ``asks'' all $n$ users if $(v,v_j)$ is the labelled edge in that level, recalling that each vertex in a hidden layer specifies a single next edge to follow. To ``answer'', each user $x_i$ compares the level $\ell$ of vertex $v$ and edge $(v,v_j)$ to their own data and replies as follows using randomized response. If $x_{i,1} = \ell$ and $(v, v_j) \in x_{i,2}$, i.e. user $i$'s hidden layer is $\ell$, and $v$ is labelled with the $(v,v_j)$ edge, then the user ``votes yes'' and transmits a draw from $\Ber{e^{\eps'}/(e^{\eps'}+1)}$, where we write $\eps' = \eps/2$ for notational simplicity. Otherwise the user ``votes no'' and transmits a draw from $\Ber{1/(e^{\eps'}+1)}$. Based on the responses, the protocol then chooses an edge out of $v$ to follow, to obtain the next vertex in the path at level $\ell+1$. When the protocol reaches a leaf, it proposes this leaf as the solution. Pseudocode for \fiub~appears in Section~\ref{subsec:fiub_pseudo}.

Intuitively, during this process, the selection made by the protocol at any level that does not correspond to a hidden layer is irrelevant: it can follow any outgoing edge and still be on track to correctly solve the problem instance. To argue for correctness, all that is important is that for the two (unknown) levels that correspond to hidden layers, the protocol correctly identifies the correct labeled edge. At each of those levels, the bias induced by randomized response will be enough to identify the correct edge with high probability. While this protocol is run, each user answers a very large (triply exponentially many in $k$) number of queries: but for only one of those queries is their response sampled from $\Ber{e^{\eps'}/(e^{\eps'}+1)}$. For all other queries, their response is sampled from $\Ber{1/(e^{\eps'}+1)}$. Hence the privacy loss over the whole protocol is constant, and does not accumulate with the number of queries. In contrast, a similar sequentially interactive protocol would require new users for each of this large number of queries.

Note that an implication of \cite{JMNR19} is that any fully interactive protocol that uses $r$ local randomizer calls per user can be converted into a sequentially interactive protocol with an $O(r)$ factor blowup in sample complexity. Consequently, it is necessary that any protocol witnessing an exponential separation between the sample complexities of fully and sequentially interactive protocols must make at least exponentially many queries per user.

\begin{theorem}
\label{thm:fi_ub}
	$\fiub$ is $(\eps,0)$-locally private and has constant success probability on $\hl(k)$  given $n = \Omega\left(\tfrac{k}{\eps^2}\right)$ samples.
\end{theorem}
\begin{proof}
	We first prove privacy. Recall that each user draws one of two samples, $(a,f)$ or $(b,g)$. Accordingly, each user has only a single point in the transcript where they output a sample from $\Ber{\tfrac{e^{\eps'}}{e^{\eps'}+1}}$, and the remainder of their outputs are all samples from $\Ber{\tfrac{1}{e^{\eps'}+1}}$. Thus, if we compare the transcript distributions (restricted to a single user with the specified data) of $\pi(a,f)$ and $\pi(b,g)$, there are at most two points in the transcript where their output distributions are not identical. Therefore for any single-user transcript output $z$, $$\tfrac{\P{}{\pi(a,f) = z}}{\P{}{\pi(b,g) = z}} \leq \tfrac{\tfrac{e^{\eps}}{(e^{\eps'}+1)^2}}{\tfrac{1}{(e^{\eps'}+1)^2}} \leq e^{\eps}.$$ $\fiub$ is thus $(\eps,0)$-locally private.
	
To reason about accuracy, we use the following simple result for randomized response.
	\begin{claim}
	\label{claim:hist}
		Let $x_1, \ldots, x_n \in \{0,1\}$ and for each $x_i$ draw $y_i \sim \Ber{e^{x_i \eps}/(e^\eps+1)}$. Let $y = \tfrac{1}{n} \cdot \sum_{i=1}^n x_i$ and $\bar y = \tfrac{1}{n} \cdot \tfrac{e^\eps+1}{e^\eps-1} \cdot \left(\sum_{i=1}^n y_i - \tfrac{n}{e^\eps+1}\right)$. Then with probability $\geq 1 - \beta$, $$|y - \bar y| \leq \tfrac{\eps + 2}{\eps\sqrt{2}} \sqrt{\ln(4/\beta)/n}.$$
	\end{claim}
	\begin{proof}
		$\E{}{\sum_i y_i} = \tfrac{y e^\eps + (n-y)}{e^\eps + 1} = \tfrac{y(e^\eps-1) + n}{e^\eps+1}$, so two Hoeffding bounds implies that with probability $\geq 1-\beta$, $$\left|\sum_i y_i -  \tfrac{y(e^\eps-1) + n}{e^\eps+1}\right| \leq \sqrt{n\ln(4/\beta)/2}.$$ Thus $$|y - \bar y| \leq \tfrac{e^\eps+1}{e^\eps-1}\sqrt{\ln(4/\beta)/2n} < \tfrac{\eps + 2}{\eps\sqrt{2}}\sqrt{\ln(4/\beta)/n}.$$
	\end{proof}
	
By Claim~\ref{claim:hist}, with probability $\geq 1-\beta/2$, whenever $\ell \in \{a,b\}$ but $(v,v_j) \not \in f \cup g$, i.e. whenever the current layer is hidden but $v_j$ is the wrong child node, $$\bar y \leq \tfrac{\eps' + 2}{\eps'\sqrt{2}}\sqrt{\tfrac{\ln(2^{4k+3}/\beta)}{n}} < \tfrac{\eps' + 2}{\eps'\sqrt{2}} \sqrt{\tfrac{(4k+2) + \ln(1/\beta)}{n}} < 0.1$$ where the last inequality uses $n > 100 \left(\tfrac{\eps' + 2}{\eps'\sqrt{2}}\right)^2 (4k+2 + \ln(1/\beta))$.

Now consider the situation where $\ell = a$ ($\ell = b$ is symmetric) and $(v,v_j) \in f$, i.e. the current layer is hidden and $v_j$ is the correct child node. Let $A = \{i \mid x_i = (a,f)\}$. First, by a Hoeffding bound with probability $\geq 1 - \beta/4$, $|A| \geq \tfrac{n}{2} - \sqrt{n\ln(4/\beta)/2}$. Thus by Claim~\ref{claim:hist} $$\bar y \geq \tfrac{1}{2} - \sqrt{\tfrac{\ln(4/\beta)}{2n}} - \tfrac{\eps' + 2}{\eps'\sqrt{2}} \sqrt{\tfrac{(4k+2) + \ln(1/\beta)}{n}} > 0.2$$ where the last inequality uses the above lower bound on $n$ and $n > 25\ln(4/\beta)$. By a union bound, with probability $\geq 1 - \beta$ when $(v,v_j) \in f$ then $\bar y > 0.2$ and $v_j$ is correctly chosen as the next child.
\end{proof}

Combining Corollary~\ref{cor:lb} and Theorem~\ref{thm:fi_ub} yields an exponential in $k$ separation between sequentially and fully interactive protocols achieving constant success probability on $\hl$. 
\section{Second Reduction:  Round-Specific Separation}
\label{sec:round_specific}
The equivalence given in Theorem~\ref{thm:eq} has one drawback: it does not preserve round complexity\footnote{In particular, the conversion between noiseless and noisy communication blows up round complexity in both directions.}. As a result, we cannot use it to prove a separation between $(k-1)$ and $k$-round sequentially interactive protocols. To do this, we instead prove a slightly different one-way reduction with a looser dependence on $\eps$.

\subsection{Second Reduction}
\label{subsec:red}

\begin{theorem}
\label{thm:red}
	Let $\A$ be a sequentially interactive $\eps$-locally private protocol computing relation $\R$ with error probability $\gamma = \Omega(1)$, and let $0 < \eta = \Omega(1)$ such that $\gamma + \eta < 1$. Then there exists a two-party protocol $\A'$ computing $\R$ with error probability $\leq \gamma + \eta$ such that $\cc(\A') = \tilde O\left(e^\eps \cdot \sac(\A)\right)$. Moreover, $\A$ and $\A'$ have identical round complexity.
\end{theorem}
\begin{proof}
	We start by using Lemma~\ref{lem:one_bit} to transform $\A$ into a different protocol $\A_1$ where each user sends a single bit. $\A_1$ has the same round complexity as $\A$ but incurs an $O(e^\eps)$ blowup in expected sample complexity. As before, we transform this into a bound that holds with probability $1-\eta$ at the cost of an increase in the failure probability by $\eta$ and a constant $1/\eta$ blowup in sample complexity.

	Next, we transform $\A_1$ into a two-party protocol $\A'$ for $\R$. To do so, we run $\A_1$ and, for each round of randomizer calls, uniformly at random assign (i.e. using public randomness generated ahead of time) each randomizer call $R$ to one of Alice and Bob for execution on their own data. Alice and Bob then release the corresponding outputs simultaneously, and we proceed to the next round of $\A_1$. Since $\A_1$ is sequentially interactive, each randomizer call is made using an unknown new sample $s$. Furthermore, since each new sample is drawn from a uniform distribution as $s \sim_U \{x,y\}$, $R(s)$ is distributed identically to $R$ executed on a uniformly random choice of Alice and Bob. At the conclusion of $\A'$, Alice and Bob mimic the output of the analyst. Since this choice depends entirely on the communications thus far, which are distributed identically in $\A_1$ and $\A'$, the distribution over final answers is identical as well. Since $\A_1$ used $O(e^\eps \cdot |\sac(\A)|)$ bits of communication, $\A'$ does as well. Finally, since each round of $\A_1$ leads to exactly one round of $\A'$, $\A$ and $\A'$ have identical round complexity.
\end{proof}

\section{Round-specific Separation for Sequential Interactivity}
\label{sec:round_sep}
In this section, we introduce the \emph{pointer chasing} problem and combine it with Theorem~\ref{thm:red} to prove a round-specific exponential sample complexity separation for sequentially interactive local privacy.

\subsection{The Pointer Chasing Problem}
\label{subsec:pointer_chasing}
In the pointer chasing problem $\pcp(k, \ell)$, Alice and Bob respectively receive vectors $a$ and $b$ in $[\ell]^\ell$ representing a sequence of pointers to locations in the other vector. The goal is to follow the chain of pointers starting from $a[1]$ and determine the value of the $k$'th pointer in this chain, where we focus on problem instances with $k \ll \ell$. The difficulty is that while Alice knows $a[1]$, to compute the third pointer location she needs to know $b[a[1]]$. Pointer-chasing is therefore a natural candidate for forcing interaction for nontrivial --- i.e., communication $o(k \log(\ell))$ --- solutions. An illustration appears in Figure~\ref{fig:pc}.

\begin{figure}[h]
\centering

\tikzset{every picture/.style={line width=0.75pt}} %set default line width to 0.75pt

\begin{tikzpicture}[x=0.75pt,y=0.75pt,yscale=-1,xscale=1]
%uncomment if require: \path (0,221); %set diagram left start at 0, and has height of 221

%Shape: Rectangle [id:dp13012835640896325]
\draw   (50,30) -- (90,30) -- (90,70) -- (50,70) -- cycle ;
%Shape: Rectangle [id:dp791202579838826]
\draw   (90,30) -- (130,30) -- (130,70) -- (90,70) -- cycle ;
%Shape: Rectangle [id:dp2720876921484463]
\draw   (130,30) -- (170,30) -- (170,70) -- (130,70) -- cycle ;
%Shape: Rectangle [id:dp9246688344922276]
\draw   (170,30) -- (210,30) -- (210,70) -- (170,70) -- cycle ;
%Shape: Rectangle [id:dp3104855953260879]
\draw   (210,30) -- (250,30) -- (250,70) -- (210,70) -- cycle ;
%Shape: Rectangle [id:dp31255408671425]
\draw   (250,30) -- (290,30) -- (290,70) -- (250,70) -- cycle ;
%Shape: Rectangle [id:dp3538344232012751]
\draw   (290,30) -- (330,30) -- (330,70) -- (290,70) -- cycle ;
%Shape: Rectangle [id:dp98724555511482]
\draw   (330,30) -- (370,30) -- (370,70) -- (330,70) -- cycle ;
%Shape: Rectangle [id:dp15273474446487678]
\draw   (50,150) -- (90,150) -- (90,190) -- (50,190) -- cycle ;
%Shape: Rectangle [id:dp8259391902895483]
\draw   (90,150) -- (130,150) -- (130,190) -- (90,190) -- cycle ;
%Shape: Rectangle [id:dp21262940858972335]
\draw   (130,150) -- (170,150) -- (170,190) -- (130,190) -- cycle ;
%Shape: Rectangle [id:dp11432006451031151]
\draw   (170,150) -- (210,150) -- (210,190) -- (170,190) -- cycle ;
%Shape: Rectangle [id:dp9533203010119058]
\draw   (210,150) -- (250,150) -- (250,190) -- (210,190) -- cycle ;
%Shape: Rectangle [id:dp25744195686839944]
\draw   (250,150) -- (290,150) -- (290,190) -- (250,190) -- cycle ;
%Shape: Rectangle [id:dp2564778273293111]
\draw   (290,150) -- (330,150) -- (330,190) -- (290,190) -- cycle ;
%Shape: Rectangle [id:dp013679386959314366]
\draw   (330,150) -- (370,150) -- (370,190) -- (330,190) -- cycle ;
%Straight Lines [id:da989764382483032]
\draw    (70,70) -- (348.08,149.45) ;
\draw [shift={(350,150)}, rotate = 195.95] [fill={rgb, 255:red, 0; green, 0; blue, 0 }  ][line width=0.75]  [draw opacity=0] (8.93,-4.29) -- (0,0) -- (8.93,4.29) -- cycle    ;

%Straight Lines [id:da7494969114336023]
\draw    (350,150) -- (231.66,71.11) ;
\draw [shift={(230,70)}, rotate = 393.69] [fill={rgb, 255:red, 0; green, 0; blue, 0 }  ][line width=0.75]  [draw opacity=0] (8.93,-4.29) -- (0,0) -- (8.93,4.29) -- cycle    ;

%Straight Lines [id:da8202583468481155]
\draw    (230,70) -- (111.66,148.89) ;
\draw [shift={(110,150)}, rotate = 326.31] [fill={rgb, 255:red, 0; green, 0; blue, 0 }  ][line width=0.75]  [draw opacity=0] (8.93,-4.29) -- (0,0) -- (8.93,4.29) -- cycle    ;

%Straight Lines [id:da9778090328079849]
\draw    (110,150) -- (110,72) ;
\draw [shift={(110,70)}, rotate = 450] [fill={rgb, 255:red, 0; green, 0; blue, 0 }  ][line width=0.75]  [draw opacity=0] (8.93,-4.29) -- (0,0) -- (8.93,4.29) -- cycle    ;

%Straight Lines [id:da4672636943500874]
\draw    (110,70) -- (268.21,149.11) ;
\draw [shift={(270,150)}, rotate = 206.57] [fill={rgb, 255:red, 0; green, 0; blue, 0 }  ][line width=0.75]  [draw opacity=0] (8.93,-4.29) -- (0,0) -- (8.93,4.29) -- cycle    ;

% Text Node
\draw (27,48.46) node   {$a$};
% Text Node
\draw (27,171.54) node   {$b$};
% Text Node
\draw (70,50) node   {$8$};
% Text Node
\draw (110,50) node   {$6$};
% Text Node
\draw (150,50) node   {$5$};
% Text Node
\draw (190,50) node   {$1$};
% Text Node
\draw (230,50) node   {$2$};
% Text Node
\draw (270,50) node   {$4$};
% Text Node
\draw (310,50) node   {$3$};
% Text Node
\draw (350,50) node   {$7$};
% Text Node
\draw (70,170) node   {$1$};
% Text Node
\draw (110,170) node   {$2$};
% Text Node
\draw (150,170) node   {$4$};
% Text Node
\draw (190,170) node   {$6$};
% Text Node
\draw (230,170) node   {$7$};
% Text Node
\draw (270,170) node   {$8$};
% Text Node
\draw (310,170) node   {$3$};
% Text Node
\draw (350,170) node   {$5$};

\end{tikzpicture}

\caption{An instance of pointer chasing $\pcp(5,8)$ with solution 8.}
\label{fig:pc}
\end{figure}
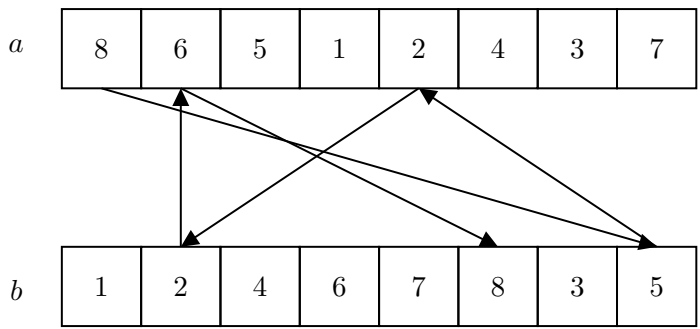

\citet{NW93} proved the following communication lower bound for $\pcp(k,\ell)$.

\begin{lemma}[Theorem 2.6 in~\citet{NW93}]
\label{lem:pc_lb}
	Let $k < \tfrac{\ell}{\log(\ell)}$. For any $k$-round protocol $\A$ with success probability $\geq 2/3$ on $\pcp(k,\ell)$ where Bob speaks first, $\cc(\A) = \Omega\left(\tfrac{\ell}{k^2}\right)$.
\end{lemma}

One obstacle remains before combining Theorem~\ref{thm:red} with Lemma~\ref{lem:pc_lb}. The transformation given in Theorem~\ref{thm:red} assumes that each round of communication features Alice and Bob releasing outputs \emph{simultaneously}. In contrast, the model used for Lemma~\ref{lem:pc_lb} supposes that Alice and Bob speak in \emph{alternating} rounds. We address this gap in Lemma~\ref{lem:diff_rounds}.

\begin{lemma}
\label{lem:diff_rounds}
	In the two-party communication model, any $(k-1)$-round simultaneous communication protocol can be simulated by a $k$-round alternating communication protocol with Bob speaking first.
\end{lemma}
\begin{proof}
	In the simultaneous protocol, let $(A_1, B_1), (A_2, B_2), \ldots$ denote the outputs produced in each round. Since the outputs are simultaneous, each $A_t$ depends only on $(A_1, B_1), \ldots, (A_{t-1}, B_{t-1})$ and Alice's data, and a similar dependence holds for $B_t$ and Bob's data. Alice and Bob can therefore simulate a simultaneous communication protocol with alternating communication protocol as follows: Bob begins by outputting $B_1$, then Alice outputs $A_1$ and $A_2$, then Bob outputs $B_2$ and $B_3$, and so on. This induces an identical distribution over final outputs $A_{k-1}$ and $B_{k-1}$ at the expense of an additional round.
\end{proof}

Combining Lemma~\ref{lem:pc_lb} with Theorem~\ref{thm:red} yields Corollary~\ref{cor:pc_ni_lb}.

\begin{corollary}
\label{cor:pc_ni_lb}
	Let $\A$ be a $(k-1)$-round sequentially interactive $\eps$-locally private protocol solving $\pcp(k,\ell)$ with error probability $\gamma \leq 1/3 - c$ for $c = \Omega(1)$. Then $\sac(\A) = \Omega\left(\tfrac{\ell}{e^\eps k^2}\right)$.
\end{corollary}

\subsection{$k$-round Upper Bound}
\label{subsec:si_ub}
We now provide a $k$-round sequentially interactive protocol, $\siub$, achieving much better sample complexity on $\pcp$. $\siub$ is a straightforward private version of the trivial $O(k\log(\ell))$ communication solution to $\pcp$ where Alice and Bob  alternate communicating the next pointer in sequence, bit by bit. Local privacy forces this communication to pass through randomized response for each bit, and sequential interaction forces the use of a new group of users for each bit. Pseudocode for $\siub$ appears in Appendix~\ref{subsec:siub_pseudo}. As both the privacy and accuracy proofs for Theorem~\ref{thm:si_ub} are nearly identical to those of Theorem~\ref{thm:fi_ub}, we defer them to Appendix~\ref{subsec:siub_proof}.

\begin{theorem}
\label{thm:si_ub}
	$\siub$ is $(\eps,0)$-locally private, solves $\pcp(k, \ell)$ with error $\leq 1/6$ and has $\sac(\A) = \tilde O\left(\tfrac{k\log(\ell)}{\eps^2}\right)$.
\end{theorem}

In conjunction with Corollary~\ref{cor:pc_ni_lb}, Theorem~\ref{thm:si_ub} yields an exponential separation in $\ell$ between $(k-1)$- and $k$-round sequentially interactive locally private protocols for $\pcp(k,\ell)$.

\section{Discussion and Open Problems}
Using tools from communication complexity, we have exhibited an exponential sample complexity separation between the sequentially interactive and fully interactive models of local differential privacy. This raises several interesting questions.
\begin{enumerate}
\item It has been known since \citet{KLNRS11} that the query complexity of locally private protocols is polynomially related to the query complexity of algorithms in the statistical query model from learning theory \cite{Kearns98}, for which there are a number of lower bounds. Since query complexity and sample complexity are equivalent for sequentially interactive protocols, and since no separation between the two models was known until very recently \cite{JMNR19}, this was widely viewed as an equivalence that might also hold for sample complexity. Now that we know there is in fact an exponential gap between the sample complexity of fully interactive protocols and sequentially interactive protocols (and hence the query complexity of algorithms in the statistical query model), this again raises the question of the power of locally private learning when protocols are not restricted to be sequentially interactive. In particular, is there an algorithm that can learn parity functions in $d$ dimensions with sample complexity that is sub-exponential in $d$? We conjecture that the answer remains ``no'' but don't know how to prove it.

\item Might there be larger than exponential separations between sequentially and fully interactive locally private protocols? Our results leave this possibility open. To refute it, one would like to demonstrate a reduction that takes as input a fully interactive protocol, and outputs a sequentially interactive protocol inducing the same transcript distribution with ``only'' an exponential increase in sample complexity. \citet{JMNR19} give a reduction in this style that is tight in terms of what they call the ``compositionality'' of the fully interactive protocol, but their reduction does not give any guarantees in terms of $n$.
\end{enumerate}

\newpage

\bibliographystyle{plainnat}
\bibliography{si_fi_2}

\newpage

\section{Appendix}
\subsection{Pseudocode for \fiub}
\label{subsec:fiub_pseudo}
\begin{algorithm}
\caption{\fiub}\label{alg:ub}
\begin{algorithmic}[1]
\Procedure{\fiub}{$\eps, n, \T$}
\State Initialize current node $v \gets$ root node
\State Initialize level $\ell \gets 0$
\State Set $\eps' \gets \eps/2$
\While{$\ell <= 2^{rs}-1$}
	\State Initialize NextNodeFound $\gets 0$
	\State Initialize child index $j \gets 0$
	\While{not NextNodeFound and $j <= 2^{4k}-1$}
		\For{users $i=1, 2, \ldots, n$}
			\State Initialize $b_i \gets 0$
			\If{$x_{i,1} = \ell$ and $(v, v_j) \in x_{i,2}$}
				\State User $i$ sets $b_i \gets 1$
			\EndIf
			\State User $i$ publishes $y_i \sim \Ber{e^{b_i\eps'}/(e^{\eps'}+1)}$
		\EndFor
		\State $\bar y \gets \tfrac{1}{n} \cdot \tfrac{e^{\eps'}+1}{e^{\eps'}-1} \cdot \left(\sum_{i=1}^n y_i - \tfrac{n}{e^{\eps'}+1}\right)$
		\If{$\bar y > 0.2$ or $v_j = v_{2^{4k}-1}$}
			\State Set new current node $v \gets v_j$
			\State NextNodeFound $\gets 1$
		\EndIf
	\EndWhile
	\State $\ell \gets \ell + 1$
\EndWhile
\State Output $v$
\EndProcedure
\end{algorithmic}
\end{algorithm}

\newpage

\subsection{Pseudocode for \siub}
\label{subsec:siub_pseudo}
\begin{algorithm}
\caption{\siub}
\begin{algorithmic}[1]
\Procedure{\siub}{$\eps, k, \ell, m$}
\State Initialize Alice or Bob indicator $z_1 \gets 0$
\State Initialize current location $z_2 \gets 1$
\For{pointer index $j = 1, 2, \ldots, k$}
	\State Initialize next pointer $s \gets 0$
	\For{possible pointer value bits $b = 1, 2, \ldots, \lceil \log(\ell) \rceil$}
		\For{each $i$ of $m$ new users}
			\State Initialize $q_i \gets 0$
			\If{$x_{i,1} = z_1$ and bit $b$ of $x_{i,2}[z_2]$ is $1$}
				\State User $i$ sets $q_i \gets 1$
			\EndIf
			\State User $i$ outputs $y_i \sim \Ber{e^{q_i\eps}/(e^\eps+1)}$
		\EndFor
		\State $\bar y \gets \tfrac{1}{m} \cdot \tfrac{e^\eps+1}{e^\eps-1} \cdot \left(\sum_i y_i - \tfrac{m}{e^\eps+1}\right)$
		\State $s \gets s + 2^{b-2}(\sgn(\bar y - 0.15)+1)$
	\EndFor
	\State $z_1 \gets -z_1 + 1$
	\State $z_2 \gets s$
\EndFor
\State Output $z_2$
\EndProcedure
\end{algorithmic}
\end{algorithm}

\subsection{Proof of Theorem~\ref{thm:si_ub}}
\label{subsec:siub_proof}
\begin{proof}
	First, we prove privacy. Since $\siub$ is sequentially interactive, each user produces a single output. As each output is produced by $\eps$-randomized response, i.e. drawn from either $\Ber{\tfrac{e^\eps}{e^\eps+1}}$ or $\Ber{\tfrac{1}{e^\eps+1}}$, $\siub$ is $\eps$-locally private.
	
	We now prove accuracy. For each pointer index in $[k]$, the protocol computes the next location in $\ell$ bit by bit, using $\lceil \log(\ell) \rceil$ queries, each to a new group of $m$ users. By a similar invocation of Claim~\ref{claim:hist} as in the proof of Theorem~\ref{thm:fi_ub}, given group size $$m > 100\left(\tfrac{\eps + 2}{\eps\sqrt{2}}\right)^2(\ln(k\lceil\log(\ell)\rceil) + \ln(2/\beta)),$$ with probability $1-\beta$ all $k\lceil\log(\ell)\rceil$ queries return the correct bit. $\siub$ thus uses $mk\lceil \log(\ell) \rceil = O\left(\tfrac{k\log(\ell)\log(k\log(\ell))}{\eps^2}\right)$ samples to solve $\pcp(k,\ell)$ with error probability $\leq 1/6$.
\end{proof}

\end{document}